\pdfoutput=1
\documentclass[final,12pt]{clear2023} %

\title[On the Interventional Kullback-Leibler Divergence]{On the Interventional Kullback-Leibler Divergence}
\usepackage{times}

\usepackage{tikz}
\usetikzlibrary{positioning,shapes,arrows}

\DeclareMathOperator*{\argmin}{arg\,min}
\DeclareMathOperator\supp{supp}
\usepackage{mathrsfs}
\usepackage{wrapfig}
\usepackage[normalem]{ulem}
\usepackage{stackrel}
\newtheorem{example}[theorem]{Example}
\newtheorem*{lemma*}{Lemma}

\clearauthor{%
 \Name{Jonas Wildberger} \Email{wildberger.jonas@tuebingen.mpg.de}\\
 \addr Max Planck Institute for Intelligent Systems,  Max-Planck-Ring 4, 72076 T\"ubingen, Germany
 \AND
 \Name{Siyuan Guo} \Email{syg26@cantab.ac.uk} \\
 \addr Max Planck Institute for Intelligent Systems,  Max-Planck-Ring 4, 72076 T\"ubingen, Germany \\
 University of Cambridge, Cambridge
 \AND
 \Name{Arnab Bhattacharyya} \Email{arnabb@nus.edu.sg}\\
 \addr School of Computing, National University of Singapore, Singapore
 \AND
 \Name{Bernhard Schölkopf} \Email{bs@tuebingen.mpg.de} \\
 \addr Max Planck Institute for Intelligent Systems,  Max-Planck-Ring 4, 72076 T\"ubingen, Germany%
}

\DeclareMathOperator{\Ex}{\ensuremath{\mathbb{E}}}
\newcommand{\KL}{D_{\mathrm{KL}}}
\newcommand{\IKL}{D_{\mathrm{IKL}}}
\newcommand{\TV}{D_{\mathrm{TV}}}

\newtheorem{assumption}[theorem]{Assumption}

\newcommand{\indep}{\perp\!\!\!\!\perp} 

\begin{document}

\maketitle

\begin{abstract}
Modern machine learning approaches excel in static settings where a large amount of i.i.d.\ training data are available for a given task. In a dynamic environment, though, an intelligent agent needs to be able to transfer knowledge and re-use learned components across domains. It has been argued that this may be possible through causal models, aiming to mirror the modularity of the real world in terms of independent causal mechanisms. However, the true causal structure underlying a given set of data is generally not identifiable, so it is desirable to have means to quantify differences between models (e.g., between the ground truth and an estimate), on both the observational and interventional level.

In the present work, we introduce the Interventional Kullback-Leibler (IKL) divergence to quantify both structural and distributional differences between models based on a finite set of multi-environment distributions generated by interventions from the ground truth.  Since we generally cannot quantify all differences between causal models for every finite set of interventional distributions, we propose a sufficient condition on the intervention targets to identify subsets of observed variables on which the models provably agree or disagree.
\end{abstract}


\section{Introduction}
Classical machine learning methods are concerned with learning a distribution $P$ (or properties thereof) from a large i.i.d.\ sample using a simpler model %
$Q$. Recent advances in causality have shown how access to heterogeneous data from multiple environments can help uncover causal relationships within the data and address problems like adversarial robustness or domain generalization \citep{causality_peters,domain_adaption_zhang,self-supervised-julius,towards_causal,probable_domain_cian, 2022arXiv220315756G}. In this framework, distributional shifts emerge from interventions in some underlying causal model $\mathscr P$. Identifying the true causal model, however, is difficult even in the multi-environment regime, requiring many environments and (often unverifiable) assumptions on the type of interventions present. One thus often needs to be satisfied with an approximate model $\mathscr Q$, which is \emph{close} to a hypothetical ground truth $\mathscr{P}$.

In the present work, we propose a quantitative notion of observational and interventional \emph{closeness} between two causal models $\mathscr P$ and $\mathscr Q$. Extending the classical Kullback-Leibler divergence, our {\em Interventional Kullback-Leibler divergence} $\IKL(\mathscr{P}_{\mathcal E} \lVert \mathscr Q)$ is computed between a finite set of interventional distributions with known intervention targets $((P^e, \mathcal I_e))_{e\in \mathcal E}$ from $\mathscr P$, denoted as $\mathscr{P}_{\mathcal E}$, and our model estimate $\mathscr Q$, consisting of a single reference distribution estimate $Q$ and graphical model estimate $G_Q$. 
Intuitively, we can understand $\IKL$ as quantifying the average deviation between distributions under interventions and our model's predictions, had the interventions also happened on our model estimate. We thus quantify how good our model estimate is beyond observational.

Minimizing the IKL divergence over our estimate $\mathscr Q$ with respect to arbitrary interventions is equivalent to finding the true causal model $\mathscr P = \mathscr Q$. In practice, having access to only a limited amount of multi-environment distributions presents interesting challenges regarding which structural differences are identifiable. We propose a sufficient condition on the intervention targets trading-off a priori knowledge about $\mathscr P$ to establish the equivalence $\IKL(\mathscr{P}_{\mathcal E} \lVert \mathscr Q)=0  \Longleftrightarrow \mathscr P = \mathscr Q$. If this condition is not known to be satisfied, we further show which differences with respect to subsets of variables can be provably identified.

Previous approaches \citep{SHD1,SHD2, SID} define interventional distances via differences between two causal graphs. In contrast, our IKL divergence measure (1) does not require access to the true causal graph $G_P$ and (2) also takes into account distributional differences. Further, unlike \citet{acharya_learning_testing_causal, ladder}, there is no need to experimentally perform (hard-)interventions for the evaluation of our distance; instead we allow for arbitrary unknown mechanism changes, as long as the targets of the interventions are discerned.

The remainder of this work is organized as follows: In Section~\ref{sec:problem_settings}, we introduce the problem settings and the assumptions to set the stage for the following theoretical analysis.
Section \ref{sec:interv_dist} contains a discussion about related work and our main results, studying the Interventional Kullback-Leibler divergence first for the example of a known graph and then in the general setting. Finally, in Section~\ref{sec:discussion}, we discuss the operational significance of the IKL divergence and various applications related to the area of multi-environment causal learning. Summarizing, our  main contributions are:
\begin{itemize}
    \item We introduce the Interventional Kullback-Leibler divergence $\IKL$ that quantifies distributional and structural differences between a given set of multi-environment distributions $(P^e)_{e\in \mathcal E}$ from the ground truth $\mathscr P$ and our inferred model $\mathscr Q$.
    \item We discuss various properties and applications of the IKL divergence related to tasks in causal inference. 
    \item We prove a theorem establishing equivalence between $\IKL(\mathscr{P}_{\mathcal E}\lVert \mathscr{Q}) = 0$ and equality between causal models $\mathscr P = \mathscr Q$ under certain conditions on the environments. We thereby trade off (partial) knowledge about $\mathscr P$ and access to interventional distributions $(P^e)_{e\in \mathcal E}$. This gives rise to a sufficient condition that can be verified using only the inferred model $\mathscr Q$.
\end{itemize}

\section{Problem Setting and Preliminaries}
\label{sec:problem_settings}

We begin by reviewing the causal formalism used below. This is based on interventional distribution shifts, modelled as causal graphical models \citep{causality_pearl}. 
\subsection{Causal Terminology}
\begin{definition}[Causal Graphical Model]
    A Causal Graphical Model (CGM) $(P, G_P)$ over a variable set ${\bf X} = \{X_1, \ldots, X_d\}$ consists of a distribution $P$ and a directed acyclic graph (DAG) $G_P = ({\bf V}, {\bf E})$ where ${\bf V} := \{1, \ldots, d\}$ is an index set and ${\bf E} \subseteq {\bf V}^2$ is an edge set, such that $(i, j) \in {\bf E}$ \text{if and only if} $X_i\text{ directly causes } X_j$. 
    We say that two CGMs $(P, G_P), (Q, G_Q)$ are the same $(P, G_P) = (Q, G_Q)$ if and only if $P=Q$ and $G_P = G_Q$.
\end{definition}

\begin{definition}[Markov Factorization]
\label{def:markov_factorization}
Given a joint distribution $P$ and a DAG $G$, we say $P$ satisfies the Markov factorization property\footnote{We assume that all distributions have densities with respect to the Lebesgue measure.} with respect to $G$ (or is Markovian with respect to $G$) if 
\begin{equation}
\label{eq:markovfactorization}
    P(X_1, \ldots, X_d) = \prod_i P(X_i \ | \ {\bf PA}_i^G ),
\end{equation}
where ${\bf PA}_i^G$ denotes the set of parents of $X_i$ in $G$.
\end{definition}

\noindent While Definition \ref{def:markov_factorization} allows us to deduce properties of the distribution $P$ from the corresponding graph $G_P$, the following faithfulness property allows us to perform inference in the converse direction, i.e.\ finding the structure of the graph given the distribution. 
\begin{definition}[Causal Faithfulness]\label{assum:faithful}
    A joint distribution $P$ is called faithful with respect to a DAG $G$ if any conditional independence relationship in $P$ is implied by $d$-separation in $G$ \citep{spirtes_1993,causality_peters}.
\end{definition}
\begin{assumption}
    In the following, we consider the set of distributions that are both Markovian and faithful with respect to some causal DAG $G$, i.e. given an observed distribution $P$, there exists a compatible graph $G_P$.
\end{assumption}
\noindent  A justification for the faithfulness assumption is given by \citet[Theorem 3.5]{Koller2010} stating that almost all distributions $P$ that are Markovian with respect to $G$ are also faithful to $G$.

\subsection{Multi-environment learning} Throughout this work, we take advantage of data from multiple environments.
Specifically, we assume that given a set of multi-environment distributions $(P^e)_{e \in \mathcal{E}}$ over some variable set ${\bf X}$, there exists an underlying ground truth CGM $\mathscr P := (P, G_P)$ giving rise to  $(P^e)_{e \in \mathcal{E}}$. The generative process behind these distributions is assumed to satisfy the below \citep{schoel_causal_and_anticausal,causality_peters}:

\begin{assumption}[Independent Causal Mechanisms (ICM)]\label{assum:ICM}
A causal generative process of a system of variables is composed of autonomous modules that do not inform or influence each other. Mathematically speaking, given a Markov factorization of the joint distribution as in \eqref{eq:markovfactorization}, the causal conditionals (also called {\em mechanisms}) $P(X_i \ | \ {\bf PA}_i^{G_P} )$ represent autonomous modules in the sense that 
\begin{enumerate}
    \item[(i)] intervening on one $P(X_i \ | \ {\bf PA}_i^{G_P} )$ will not change other mechanisms $P(X_j \ | \ {\bf PA}_j^{G_P} ), i\neq j$,
    \item[(ii)] knowing one mechanism will not provide information about other mechanisms. 
\end{enumerate}
\end{assumption}

\noindent Under Assumption~\ref{assum:ICM}, environment shifts act independently on each of the mechanisms $P(X_i \ | \ {\bf PA}_i^{G_P})$. Thus, for each environment, $e\in {\mathcal E}$ there exists a well-defined set of mechanisms that are altered by the environment.
\begin{assumption}[Multi-Environment Distributions] \label{assum:multi-env}
For each environment $e \in \mathcal{E}$, `Nature' first chooses a subset of variables to intervene upon $\mathcal{I}_e \subseteq [d]$, where $[d]:= \{1, \dots, d\}$. For each intervened variable, `Nature' then chooses a mechanism shift $\tilde P(X_i \ | \  {\bf PA}_i) \in \mathcal{F}$ to perform, where $\mathcal F$ is a family of functions.
The transition from $P$ to $P^e$ thus maps $P(X_i \  | \ {\bf PA}_i) \mapsto \tilde{P}(X_i \ | \ {\bf PA}_i)$ if $i \in \mathcal I_e$ and $P(X_i \ | \ {\bf PA}_i) \mapsto P(X_i \ | \ {\bf PA}_i)$ otherwise, and we may represent the joint distribution in environment $e$ as:
    \begin{equation}
        P^e(X_1, \ldots, X_d) = \prod_{i \in \mathcal I^e} \tilde P(X_i \ | \ {\bf PA}_i) \cdot \prod_{i \in [d] \backslash \mathcal I^e} P(X_i \ | \ {\bf PA}_i) 
    \end{equation}
\end{assumption}

\noindent Such mechanism shifts arise when the context of the observed distribution changes. This may happen for example as the result of different outer circumstances like climate or time \citep{climate_time} or by explicitly intervening on an experiment. Such shifts are generally called \emph{interventions}. Special cases include \emph{soft interventions} where the structure of the graph is left invariant, i.e., $X_i$ does not become conditionally independent of any of its parents, and \emph{hard interventions} where $\tilde P(X_i \ | \ {\bf PA}_i) = \delta(X_i - x)$. In any case, $P^e$ will still be Markovian with respect to the ground truth graph $G_P$. However, it may no longer be faithful to it, unless the interventions are soft. 

Our final assumption concerns the existence of hidden confounder variables. We here relax the causal sufficiency assumption (stating that there are no unobserved confounders) to pseudo causal sufficiency. This ensures in particular that there are no changes in observed distributions that are not explained by the environment and the causal graph $G_P$.

\begin{assumption}[Pseudo causal sufficiency]\label{asum:pseudo-causal}
    The effect of any unobserved confounder can be completely explained by the environment variable, i.e.\ in any given environment $e \in \mathcal E$ potential unobserved confounders are fixed, and causal sufficiency is satisfied \citep{Huang_causal_discovery_from_het}.
\end{assumption}

\section{The Interventional KL divergence} \label{sec:interv_dist}
In this section, we introduce our proposed Interventional KL divergence. We begin by reviewing the problem of defining a distance between two CGMs $\mathscr P = (P, G_P)$ and $\mathscr Q =(Q, G_Q)$ and discussing related work. We then introduce the IKL divergence first for a known true causal graph $G_P$. This definition is subsequently generalized to an unknown $G_P$. Finally, we establish the equivalence $\mathscr P = \mathscr Q \Longleftrightarrow \IKL({\mathscr P}_{\mathcal E} \lVert \mathscr Q) = 0$, and derive a corollary in a partial information setting.
Thereby, $\mathscr Q$ takes the role of an `estimate' for the unknown $\mathscr P$, i.e.\ we know $G_Q$ and have access to and can evaluate the density of $Q$. For complete proofs of the results in this section, we refer to Appendix~\ref{appendix:proofs}. 
\subsection{Related Work}
\label{subsec:related_work}
Defining a general distance between two causal graphical models $(P, G_P), (Q, G_Q)$ requires comparing them both with respect to distributional differences between $P, Q$ and structural differences between their underlying causal graphs $G_P, G_Q$. Early work addressing this problem considered only structural differences. The Structural Hamming Distance (SHD) \citep{SHD1, SHD2} counts the number of different edges between two graphs, which has been found to provide limited insights about the effect on interventional distributions \citep{SID}. They instead proposed the Structural Intervention Distance (SID) to address this limitation by counting instead the number of different interventional distributions induced by two causal graphs. If the true causal graph $G_P$ is unknown, it is generally unclear how to define a valid distance measure between the causal models. One avenue to solve this problem is leveraging multi-environment distributions that
under Assumptions \ref{assum:multi-env} and \ref{asum:pseudo-causal} 
can provide insights about the structural properties of the true causal graph. \cite{ladder, acharya_learning_testing_causal} propose a score for comparing general causal graphical models based on their distributional distance after performing hard interventions on both of them. In section \ref{sec:interventional-unknown}, we show that under these assumptions, the Interventional KL divergence entails their distance as a special case.

Another related area of research is that of multi-environment causal discovery, which seeks to uncover the true causal graph $G_P$ by relaxing the i.i.d.\ assumption to exchangeable data \citep{2022arXiv220315756G} or interventional data \citep{pmlr-v2-eaton07a, NEURIPS2018_6ad4174e, multi-env-unknown-he-geng}. Since these methods usually only provide probabilistic guarantees about recovering the true causal graph, a causal distance metric can be seen as an evaluation tool to track their progress.

Finally, there is related work discussing the identifiability of causal effects and interventional Markov equivalence classes \citep{tian_pearl, hauser_buehlmann_imarkov, hauser_2015_imarkov}. 
Our work is tangent to this line of research in that we also seek to understand which kind of interventions are necessary in a general multi-environment setting to distinguish causal models and define a valid distance measure between them.

\subsection{The IKL divergence for a known causal graph}
Before defining the Interventional KL divergence for general causal models $\mathscr P = (P, G_P) , \mathscr Q = (Q, G_Q)$ in the next section, we build intuition by first discussing the case of a known ground truth causal graph $G_P$ of $\mathscr{P}$. This means that $G_Q = G_P$, and $Q$ is Markovian with respect to $G_P$. Furthermore, any differences between $\mathscr P, \mathscr Q$ are given by distributional differences between $P,Q$.

Throughout this work, we use the Kullback-Leibler divergence for quantifying distributional differences between $P$ and $Q$, assuming that $P$ is absolutely continuous with respect to $Q$: 
\begin{equation}
    \KL(P({\bf X}) \lVert Q({\bf X})) = \int_{\mathcal{X}} p({\bf x}) \log \frac{p({\bf x})}{q({\bf x})} d{\bf x},
\end{equation}
It is known that $\KL(P({\bf X}) \lVert Q({\bf X})) = 0$ if and only if $P=Q$. Now suppose that $P^e$ arises from $P$ via a mechanism shift on the variables ${\bf X}_{\mathcal I_e}$ for a subset $\mathcal I_e \subseteq [d]$, i.e.\ 
\begin{equation}\label{eq:change-multi}
P^e(X_i |{\bf PA}_i^{G_P}) =  P(X_i | {\bf PA}_i^{G_P}) \quad  \text{if and only if} \quad i \in [d] \backslash  \mathcal I_e.
\end{equation} 
From the chain rule for relative entropy (see e.g.\ \cite[Thm 2.5.3]{cover_thomas}, \citep{Budhathoki2021}), we can deduce the following decomposition to understand how the effect of the mechanism shifts shows up in $\KL(P^e({\bf X})\lVert Q({\bf X}))$. From here on, we denote $\mathbb{E}_{x \sim P(x)}[X]$ as $\mathbb{E}_{P(x)}[X]$. 
\begin{restatable}{lemma}{lemKLone}\label{lem:KL1}
     Given causal model $\mathscr{P} = (P, G_P)$
     assume that $P^e$ emerges from $P$ via an environment shift according to Assumption \ref{assum:multi-env}. If $Q$ is Markovian with respect to $G_P$, we have the following decomposition
    \begin{equation}
    \label{eq:kl-decomposition-obs-lemma-7}
    \begin{split} 
    \KL(P^e({\bf X}) \lVert Q({\bf X}))  &= \sum_{i\in \mathcal I_e} \Ex\limits_{P^e({\bf pa}_i^{G_P})}\left[\KL\left(P^e(X_i \ | \ {\bf pa}_i^{G_P}) \Vert Q(X_i \ |  \ {\bf pa}_i^{G_P})\right)\right] \\
    & \hspace{.5cm} +\underbrace{\sum_{i\in[d]\backslash\mathcal I_e} \Ex\limits_{P^e({\bf pa}_i^{G_P})}\left[\KL\left(P(X_i \ | \ {\bf pa}_i^{G_P}) \Vert Q(X_i \ |  \ {\bf pa}_i^{G_P})\right)\right]}_{=:\IKL(\mathscr{P}_{\{e\}}\lVert \mathscr Q)}
    \end{split}
\end{equation}
As a special case we have for $\mathcal I_e = \emptyset $:
\begin{equation}\label{eq:kl-decomposition-special-lemma-7}
    \KL(P({\bf X}) \lVert Q({\bf X}))  = \sum_{i\in [d]} \Ex\limits_{P({\bf pa}_i^{G_P})}\left[\KL\left(P(X_i \ | \ {\bf pa}_i^{G_P}) \Vert Q(X_i \ |  \ {\bf pa}_i^{G_P})\right)\right]
\end{equation}

\end{restatable}
\noindent The first sum in equation (\ref{eq:kl-decomposition-obs-lemma-7}) goes over the intervened mechanisms and therefore fully characterizes the changes that are directly due to environment shifts. Note that it vanishes if we could perfectly replicate the environment shift on our model estimate $Q$.
The second sum, which is a special case of the Interventional KL divergence to be defined more generally below, goes over the non-intervened variables, quantifying the distributions' distance resulting downstream from intervened variables due to the fact that the expectations over parents are taken with respect to intervened distributions.

Now assume that $P$ is unobserved, but we still want to understand whether $Q$ is a valid model. Since, $P, Q$ are Markovian with respect to the same causal DAG $G_P$, we can express agreement between $P$ and $Q$ equivalently in terms of vanishing Interventional KL divergence $\IKL(\mathscr{P}_{\{e\}} \lVert\mathscr Q)$. First suppose that $P =Q$. The first sum in equation (\ref{eq:kl-decomposition-obs-lemma-7}) can become arbitrarily large, inflating the KL divergence  $\KL(P^e({\bf X}) \lVert Q({\bf X}))$. 
However, the second sum in equation (\ref{eq:kl-decomposition-obs-lemma-7}) vanishes, i.e.\, $\IKL(\mathscr{P}_{\{e\}} \lVert \mathscr Q)=0$, since agreement between the conditional distributions is irrespective of the parent distributions taken in the expectation.\footnote{Specifically, they need to coincide on the union of support sets $\supp (P({\bf PA}_i^{G_P}))\cup \supp (P^e({\bf PA}_i^{G_P}))$}
The IKL divergence therefore discounts the additional difference between $P^e$ and $Q$ introduced by the environment shift. 

Conversely, assume that $\IKL(\mathscr{P}_{\{e\}}\lVert \mathscr Q)=0$. It is clear that this does not generally imply that $P=Q$, since we only sum over a subset of the local divergences composing the difference between the joint distribution $P$ and $Q$, as shown in equation (\ref{eq:kl-decomposition-special-lemma-7}).
But given multiple interventional distributions $(P^e)_{e\in \mathcal E}$, such that each mechanism constituting $P({\bf X})$ remains unchanged in (at least) one environment, we can obtain a valid distance measure between $P$ and $Q$ via averaging. These observations are formalized in the following definition and lemma. 

\begin{definition}[Interventional KL divergence for shared causal graphs] \label{def:interv-div-known}
    Let $\mathscr P = (P, G_P), \mathscr Q = (Q, G_P)$ be two causal models sharing the same causal graph. Further, assume that we have access to a set of interventional distributions $\mathscr P_{\mathcal E} = \left((P^e, \mathcal I_e)\right)_{e\in \mathcal E}$ generated from $\mathscr P$.  We define the Interventional KL divergence as
    \begin{equation}
        \begin{split}
        \IKL(\mathscr{P}_{\mathcal E}\lVert \mathscr Q) &= \frac{1}{|\mathcal E|}\sum_{e\in \mathcal E}\sum_{i \in [d] \backslash \mathcal I_e}\Ex\limits_{P^e({\bf pa}_i^{G_P})}\left[\KL\left(P^{e}(X_i \ | \ {\bf pa}_i^{G_P}) \Vert Q(X_i \ |  \ {\bf pa}_i^{G_P})\right)\right] \label{eq:interv-def-known} \\
        & \overset{(\ref{eq:change-multi})}{=} \frac{1}{|\mathcal E|}\sum_{e\in \mathcal E}\sum_{i \in [d] \backslash \mathcal I_e}\Ex\limits_{P^e({\bf pa}_i^{G_P})}\left[\KL\left(P(X_i \ | \ {\bf pa}_i^{G_P}) \Vert Q(X_i \ |  \ {\bf pa}_i^{G_P})\right)\right] 
        \end{split}
    \end{equation}
  If $\IKL(\mathscr{P}_{\mathcal E}\lVert \mathscr Q) = 0$, we say that our model $\mathscr Q$ is interventionally equivalent with $\mathscr P$ with respect to $\mathcal E$, denoted as $\mathscr P \sim_{\mathcal E} \mathscr Q$.
\end{definition}
\begin{restatable}{lemma}{lemintervequivlight}\label{lem:interv-equiv-light}
Under Assumptions~\ref{assum:multi-env}, \ref{asum:pseudo-causal} and that no variable is always intervened upon, i.e. $\bigcap_{e\in \mathcal E} \mathcal I_e = \emptyset$, we can conclude that $\mathscr P, \mathscr Q$ are interventionally equivalent, 
    $\mathscr P \sim_{\mathcal E} \mathscr Q$,
    if and only if $\mathscr P = \mathscr Q$.
\end{restatable}

\noindent Note that if the reference distribution $P({\bf X})$ is included in our set of interventional distributions, i.e.\ there exists $e \in \mathcal{E}$ such that $\mathcal I_{e} = \emptyset$, the KL divergence $\KL(P({\bf X})\lVert Q({\bf X}))$ makes up one of the terms in the IKL divergence. Thus, the equivalence $\mathscr P \sim_{\mathcal E} \mathscr Q \Longleftrightarrow \mathscr P = \mathscr Q$ is trivially satisfied. We now illustrate an application of the case  where we do not have access to the reference distribution $P$ of the CGM $(P, G_P)$.
\begin{example}[Partial Observability] 
In this example, we illustrate how the ideas presented above formalize the common conception of recovering the joint distribution $P(X_1, X_2, X_3)$ without observing all variables at once --- given that we know it factorizes in a non-trivial way. So, let the causal graph $G_P$ be given by the DAG in Figure \ref{fig:true-causal}. By the Markov factorization (\ref{eq:markovfactorization}), we know that ${P({\bf X}) = P(X_1) \allowbreak \cdot P(X_2 \ | \ X_1)   \cdot P(X_3 \ | \ X_1)}$.

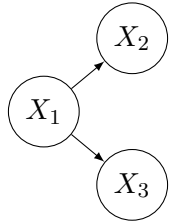
\begin{figure}[h]
   \centering
   \begin{tikzpicture}[
      node distance=0.3cm and .5cm,
      mynode/.style={draw,circle,align=center, minimum size=0.2cm}
    ]
    \node[mynode] (x1) {$X_1$};
    \node[mynode,above right=of x1] (x2) {$X_2$};
    \node[mynode,below right=of x1] (x3) {$X_3$};
    \path (x1) edge[-latex] (x2)
          (x1) edge[-latex] (x3);
    \end{tikzpicture}
     \caption{A sample DAG over variables $X_1, X_2, X_3$. } 
          \label{fig:true-causal}
\end{figure}
\noindent Now assume that we do not have access to this joint distribution, but instead we are only provided with two marginal distributions $P(X_1, X_2), P(X_1, X_3)$ and the knowledge of the underlying DAG. Even with no interventions, observing the two marginal distributions is sufficient to recover the joint distribution $P(X_1, X_2, X_3)$, since it is possible to separately estimate $P(X_2 \ | \ X_1),\allowbreak P(X_3 \ | \ X_1)$ from the marginals and $P(X_1)$ from either of the marginals. This phenomenon is correctly captured in our IKL metric by modelling unobserved variables as interventions, say $P^{e_1}({\bf X}) = P(X_1, X_2) \tilde P(X_3)$, $P^{e_2}({\bf X}) = P(X_1, X_3) \tilde P(X_2)$:
\begin{align}
    & \KL(P(X_1, X_2) \ | \ Q(X_1, X_2)) + \KL(P(X_1, X_3) \ | \ Q(X_1, X_3))  \\
    \begin{split}\label{eq:example-local-terms}
     \overset{(\ref{eq:kl-decomposition-special-lemma-7})}{=} & \ \KL(P(X_1) \Vert Q(X_1)) + \Ex\limits_{P(X_1)}[\KL(P(X_2 \ | \ X_1) \Vert Q(X_2 \ |  \ X_1)) \\
      &  \hspace{.5cm}+  \KL(P(X_1) \Vert Q(X_1)) +\Ex\limits_{P(X_1)}\left[\KL(P(X_3 \ | \ X_1) \Vert Q(X_3 \ |  \ X_1))\right]
    \end{split} \\
     \overset{(\ref{eq:interv-def-known})}{=}&  \  2 \cdot \IKL(\mathscr{P}_{\{e_1, e_2\}}\lVert \mathscr Q) 
\end{align}
We can now minimize each of the terms in equation (\ref{eq:example-local-terms}) with respect to our model distribution $Q$. By Lemma \ref{lem:interv-equiv-light}, this will also minimize the distance between the joint distributions $P, Q$ without observing all variables at once. Note that the term $\KL(P(X_1) \Vert Q(X_1))$ occurs twice in equation (\ref{eq:example-local-terms}). This means that we could still identify the joint distribution in the same way, if an intervention had happened on $X_1$ in either $P(X_1, X_2)$ or $P(X_1, X_3)$.
\end{example}

\subsection{The IKL divergence in the general case}\label{sec:interventional-unknown}
In the previous section, we discussed how the Interventional KL divergence quantifies differences between two causal models $\mathscr P = (P, G_P)$ and $\mathscr Q = (Q, G_P)$ when $G_P$ is known. Since this is generally not the case, we now generalize Definition \ref{def:interv-div-known} to an unknown $G_P$. To this end, we need to adapt Lemma \ref{lem:KL1} determining the decomposition of the KL divergence, since the conditioning sets are also unknown. As it turns out, this decomposition still holds when conditioning on ${\bf PA}_i^{G_Q}$ as long as $P$ is Markovian with respect to $G_Q$. Otherwise, there is a residual term quantifying the difference of $P$ from being Markovian, for which we introduce the notion of a \emph{Markov Projection}.

\begin{definition}[Markov Projection]
    Let $\mathcal A$ be the space of probability distributions over the variable set ${\bf X}$. The \emph{Markov projection} (or \emph{$M$-projection}) of a probability distribution $P$ onto a DAG $G$ is the mapping $\pi_G: \mathcal A \to \mathcal A$
    defined by:
    \[
    \pi_G(P) = \argmin_{Q \text{ Markov w.r.t. }G} \KL(P \lVert Q)
    \]
\end{definition}
The following Lemma \ref{lem:projection} explicitly describes the Markov projection onto a DAG $G$.
\begin{restatable}{lemma}{lemprojection}\label{lem:projection}%
Given a distribution $P$ and a DAG $G$:
    \begin{equation}
    \pi_G(P)({\bf X})=\prod_{i=1}^d P(X_i \ | \ {\bf PA}_i^G).
    \end{equation}
\end{restatable}
\noindent From the Markov factorization property \ref{def:markov_factorization} it follows that $P=\pi_G(P)$ if and only if $P$ is Markovian with respect to $G$. Otherwise, we have $P \neq \pi_G(P)$, since there are variables $X_i, X_j$ such that $X_i {\not\indep} X_j \ | \ {\bf PA}_i^G$ with respect to $P$, but $X_i \indep X_j \ | \ {\bf PA}_i^G$ in $\pi_G(P)$.
If we now replace $G_P$ by $G_Q$ in Lemma \ref{lem:KL1} and allow for general (non-Markovian) distributions $P$, the deviation from being Markovian enters the equation as an additional term: 

\begin{restatable}{lemma}{lemKLtwo}\label{lem:KL2}
    Given two CGMs $ \mathscr P = (P, G_P), \mathscr Q = (Q, G_Q)$, assume that $P^e$ emerges from $P$ via an environment shift according to Assumption \ref{assum:multi-env}. We then have the following decomposition
    \begin{equation}
    \label{eq:kl-decomposition-obs}
    \begin{split} 
    \KL(P^e({\bf X}) \lVert Q({\bf X}))  &= \sum_{i\in \mathcal I_e} \Ex\limits_{P^e({\bf pa}_i^{G_Q})}\left[\KL\left(P^e(X_i \ | \ {\bf pa}_i^{G_Q}) \Vert Q(X_i \ |  \ {\bf pa}_i^{G_Q})\right)\right] \\
    & \hspace{.5cm} +\sum_{i\in[d]\backslash\mathcal I_e} \Ex\limits_{P^e({\bf pa}_i^{G_Q})}\left[\KL\left(P^e(X_i \ | \ {\bf pa}_i^{G_Q}) \Vert Q(X_i \ |  \ {\bf pa}_i^{G_Q})\right)\right] \\
    & \hspace{.5cm} + \KL\left(P^e({\bf X}) \lVert \pi_{G_Q}(P^e)({\bf X})\right)
    \end{split}
\end{equation}
As a special case we have for $\mathcal I_e = \emptyset $
\begin{equation}\label{eq:kl-decomposition-special}
    \begin{split}
    \KL(P({\bf X}) \lVert Q({\bf X}))  &= \sum_{i\in [d]} \Ex\limits_{P({\bf pa}_i^{G_Q})}\left[\KL\left(P(X_i \ | \ {\bf pa}_i^{G_Q}) \Vert Q(X_i \ |  \ {\bf pa}_i^{G_Q})\right)\right] \\
    & \hspace{.5cm} + \KL\left(P({\bf X}) \lVert \pi_{G_Q}(P)({\bf X})\right)
    \end{split}
\end{equation}

\end{restatable}

\noindent With this generalized decomposition, we adapt Definition \ref{def:interv-div-known} of the Interventional KL divergence correspondingly to arbitrary CGMs $\mathscr P, \mathscr Q$.

\begin{definition}[Interventional KL divergence] \label{def:interv-div-unknown}
    Let $\mathscr P = (P, G_P), \mathscr Q = (Q, G_Q)$ be two causal models. Assume that we have access to a set of interventional distributions $\mathscr P_{\mathcal E} = \left((P^e, \mathcal I_e)\right)_{e\in \mathcal E}$ generated from $\mathscr P$. We then define the Interventional KL divergence as
    \begin{equation}
    \begin{split}
        \IKL(\mathscr{P}_{\mathcal E}\lVert \mathscr Q) &= \frac{1}{|\mathcal E|}\sum_{e\in \mathcal E}\Bigg[\sum_{i \in [d] \backslash \mathcal I_e}\Ex\limits_{P^e({\bf pa}_i^{G_Q})}\left[\KL\left(P^e(X_i \ | \ {\bf pa}_i^{G_Q}) \Vert Q(X_i \ |  \ {\bf pa}_i^{G_Q})\right)\right]
        \\& \hspace{2cm} + \KL\left(P^e({\bf X}) \lVert \pi_{G_Q}(P^e)({\bf X})\right)\Bigg]
        \label{eq:interv-def-unknown}
    \end{split}
    \end{equation}
\end{definition}

\noindent Note that this definition is consistent with our previous notion of interventional distance in the case where $G_P = G_Q$, see Definition \ref{def:interv-div-known}. If $G_P \neq G_Q$, Lemma \ref{lem:KL2} implies that we can capture differences between the causal models by comparing the Markov factorizations of $P^e$ and $Q$ with respect to $G_Q$ (first line in equation (\ref{eq:interv-def-unknown})), if we also account for potential deviations from $P$ satisfying this factorization (second line in equation (\ref{eq:interv-def-unknown})).

The definition of $\IKL$ went under the premise that the effect of the interventions that we consider is unknown, i.e.\ only the joint distributions $P^e({\bf X})$ and intervention targets $\mathcal I_e$ are observed, but we don't know the mapping $P(X_i \ | \ {\bf PA}_i^{G_P}) \mapsto \tilde P(X_i \ | \ {\bf PA}_i^{G_P})$. If we know the mechanism shifts that happened between the distribution $P({\bf X})$ and the interventional distributions $P^e({\bf X})$ and we are able to accurately reproduce them in our model $\mathscr Q = (Q, G_Q)$, then the $\IKL$ collapses to a simpler form, namely the average of the KL divergences between the joint distributions $P^e({\bf X})$ and $Q({\bf X})$. 

\begin{restatable}{theorem}{thmknowninterv}\label{thm:known-interv}[Known interventions]
    Given causal models $\mathscr P = (P, G_P), \mathscr Q = (Q, G_Q)$ and a set of interventional distributions $\mathscr P_{\mathcal E} = \left((P^e, \mathcal I_e)\right)_{e\in \mathcal E}$, define $Q^e$ for each environment $e\in \mathcal E$ as follows:
    \begin{equation}
        Q^e({\bf X}) = \prod_{i \not \in \mathcal I_e} Q(X_i \ | \ {\bf PA}_i^{G_Q}) \cdot \prod_{j \in \mathcal I_e} P^e(X_j \ | \ {\bf PA}_j^{G_Q})
    \end{equation}
    where $ P^e(X_j \ | \ {\bf PA}_j^{G_Q})$ denotes the mechanisms changed as the result of the environment shift. Then we have that 
    \begin{equation}
        \IKL(\mathscr{P}_{\mathcal E}\lVert \mathscr Q) = \frac{1}{|\mathcal E|} \sum_{e\in \mathcal E} \KL(P^e({\bf X}) \lVert Q^e({\bf X})) 
    \end{equation}
\end{restatable}
\begin{proof}[Sketch]
    If we apply the same mechanism shifts to $Q^e$ that appear between $P$ and $P^e$, we can use the decomposition of the KL divergence in Lemma \ref{lem:KL2} to see that these terms cancel out. The sum over the remaining terms then equals our definition of the Interventional KL divergence.
\end{proof}
\noindent In particular, this Theorem applies to hard interventions on single variables that were considered by \cite{ladder}.

\subsection{Interventional differences in the IKL divergence}
In the previous sections, we have defined the Interventional KL divergence for two general causal graphical models $\mathscr P, \mathscr Q$ and related it to the (standard) KL divergence. However, it still remains to show that the IKL divergence defines a valid distance between $\mathscr P, \mathscr Q$, i.e. that  $\mathscr P \sim_{\mathcal E} \mathscr Q \Longleftrightarrow \IKL(\mathscr P_{\mathcal E} \lVert \mathscr Q) = 0 \Longleftrightarrow \mathscr P = \mathscr Q$ under suitable conditions on $\mathcal E$. This equivalence will be subject to this section. In particular, we will focus on differences beyond Markov equivalence, since statistically identifiable differences between $\mathscr P, \mathscr Q$ are already encoded in the purely observational distance $\KL(P({\bf X})\lVert Q({\bf X}))$. For simplicity, we assume in the following that the empty intervention $\mathcal I = \emptyset$ is included in the environment.
Further, we make use of our assumptions on the multi-environment distributions (Assumptions~\ref{assum:multi-env}, \ref{asum:pseudo-causal}), which imply for Markov equivalent $G_P, G_Q$ that the following statements are equivalent:

\begin{enumerate}
     \item[(i)] For all $e$ with $i\not\in \mathcal I_e$: \ $\Ex\limits_{P^e({\bf PA}_i^{G_Q})}\left[\KL\left(P^e(X_i \ | \ {\bf PA}_i^{G_Q}) \Vert P(X_i \ |  \ {\bf PA}_i^{G_Q})\right)\right] = 0 $
    \item[(i)] ${\bf PA}_i^{G_Q} = {\bf PA}_i^{G_P}$.
\end{enumerate}
As shown in Corollary \ref{cor:identification-distance} and Example \ref{example:interventions-not-sufficient}, this equivalence does not hold given only a limited set of interventional distributions. A comprehensive characterization of differences between $P^e(X_i \ | \ {\bf PA}_i^{G_Q})$ and $P(X_i \ | \ {\bf PA}_i^{G_Q})$ can be given based on $d$-separation in the augmented DAG $G_{P, {\bf X} \cup \{e\}}$, assuming faithfulness \citep{causdiscSMS}. We adopt the faithfulness assumption in the following but, since $G_P$ is assumed to be generally unknown, we here propose a simpler sufficient condition that trades off knowledge about $G_P$ with access to interventional distributions. It is based on the following definition of (directed) unblocked paths between a variable $X_i$ and intervened variables.

\begin{definition}
    Let $\mathscr P = (P, G_P)$ be a CGM and $e\in \mathcal E$ an environment. For some  (potentially different) DAG $G$ over the same variables ${\bf X}$, we say that there exists a (directed) unblocked path $e \overset{G}{\to} X_i$ in $G$ given some variable set ${\bf Z} \subseteq {\bf X}$, if there exists $k\in \mathcal I_e$ and a directed path $X_k \to \ldots \to X_i$ in $G$ such that no element of ${\bf Z}$ intersects with the path. If $k=i$, this condition is trivially satisfied.
\end{definition}
\noindent 
So, in particular, $e \overset{G_P}{\to} X_i$ implies that the variable $X_i$ is d-connected to an intervened variable given conditioning set ${\bf Z}$. Assuming faithfulness, we can conclude that the environment has an influence on $X_i$ given ${\bf Z}$ and thus $P^e(X_i | {\bf Z}) \neq P(X_i | {\bf Z})$.
If $G_Q$ shares the same skeleton with $G_P$, this definition yields a sufficient condition for the mechanism $P^e(X_i \ | \ {\bf PA}_i^{G_Q})$ to be different from $P(X_i \ | \ {\bf PA}_i^{G_Q})$.
\begin{restatable}{lemma}{lemcondchange} \label{lem:cond-change}
    Let $\mathscr P = (P, G_P), \mathscr Q = (Q, G_Q)$ be CGMs, such that $G_P$ and $G_Q$ share the same skeleton and $X_i \overset{G_Q}{\to}X_j$, $X_i \overset{G_P}{\leftarrow} X_j$ be a flipped edge between $G_P$ and $G_Q$. Further, let $P^e$ be an interventional distribution generated from $\mathscr P$.
    \begin{enumerate}
        \item[(i)] If there exists an unblocked path $e \overset{G_P}{\to} X_i$ given ${\bf PA}_j^{G_Q} \backslash X_i$, then $P^e(X_j \ | \ {\bf PA}_j^{G_Q}) \neq P(X_j \ | \ {\bf PA}_j^{G_Q})$.
        \item[(ii)] If there exists an unblocked path $e \overset{G_P}{\to} X_j$ given ${\bf PA}_i^{G_Q}$, then $P^e(X_i \ | \ {\bf PA}_i^{G_Q}) \allowbreak \neq \allowbreak  P(X_i \ | \ {\bf PA}_i^{G_Q})$.
    \end{enumerate}
\end{restatable}
\noindent 
This Lemma yields a sufficient condition for structural differences beyond Markov-equivalence to show up in the IKL-divergence. Even if $P=Q$ coincide observationally, the interventional terms $\Ex\limits_{P^e({\bf PA}_k^{G_Q})}\left[\KL\left(P^e(X_k \ | \ {\bf PA}_k^{G_Q}) \Vert Q(X_k \ |  \ {\bf PA}_k^{G_Q})\right)\right] >0 $ for $k\in \{i,j\}$ under the conditions of this Lemma. Since these are conditions on the graphical structure of $G_P$, they cannot be verified if $G_P$ is unknown. But as it turns out, it is sufficient if such paths exist in $G_Q$.
We formalize these observations in the following Theorem:
\begin{restatable}{theorem}{thmsuffinterv}\label{thm:suff-interv}
    Let $\mathscr P = (P, G_P), \mathscr Q = (Q, G_Q)$ be CGMs and $\mathcal E$ a set of environments.
    Now assume that for all (unoriented) edges $X_i \overset{G_Q}{\to} X_j$ there exists an environment $e\in \mathcal E$ such that one of the following conditions holds:
    \begin{enumerate}
        \item[(i)] There exists an unblocked path $e \overset{G_Q}{\to} X_i$ given ${\bf PA}_j^{G_Q} \backslash X_i$ and $j \not \in \mathcal I_e$
        \item[(ii)] There exists an unblocked path $e \overset{G_Q}{\to} X_j$ given ${\bf PA}_i^{G_Q}$ and $i \not \in \mathcal I_e$
    \end{enumerate}
    Then, $\mathscr P, \mathscr Q$ are interventionally equivalent, $\mathscr P \sim_{\mathcal E} \mathscr Q$,
    if and only if $\mathscr P = \mathscr Q$.
\end{restatable}
\begin{proof}[sketch]
    If $\mathscr P = \mathscr Q$, we can deduce $\IKL(\mathscr P \lVert \mathscr Q) = 0$ from our Assumptions on the environment \ref{assum:multi-env} and pseudo-causal faithfulness \ref{asum:pseudo-causal}. Conversely, if $P \neq Q$, $\IKL(\mathscr P_{\mathcal E} \lVert \mathscr{Q}) \geq 1/|\mathcal E| \allowbreak \KL(P({\bf X})\lVert Q({\bf X})) > 0$. So it suffices to show that $P=Q$, but $G_P \neq G_Q$ implies that $\IKL(\mathscr{P}_{\mathcal E}\lVert \mathscr Q) > 0$. This case then follows from Lemma \ref{lem:cond-change}, since an unblocked path in $G_Q$ gives rise to an unblocked path to a mis-oriented edge in $G_P$ and thus one of the conditions of Lemma \ref{lem:cond-change} is satisfied, leading to a nonzero term in the IKL divergence.
\end{proof}
\noindent The less we know about the structure of $G_P$ a priori, the more interventional distributions we need to identify all structural differences between the causal models. Since interventional data can be expensive to collect, we will now show how to identify partial structural differences between $\mathscr{P}$ and $\mathscr{Q}$, given insufficient multi-environmental information to satisfy the conditions of Theorem \ref{thm:suff-interv}. In particular, we also allow for environments where only marginal information about a subset of all variables is observed.

\begin{restatable}{corollary}{coridentificationdistance}\label{cor:identification-distance}
    Let $\mathscr P = (P, G_P), \mathscr Q = (Q, G_Q)$ be CGMs with $P=Q$ and $\mathcal E$ a set of environments. Assume that we only have partial multi-environmental information in the following ways:
    \begin{enumerate}
        \item[1)] We only observe a subset of all variables in the causal graph ${\bf X}_S \subseteq \bf X$, i.e.\ we can only distinguish the graphical sub-models  $\mathscr P\big|_S, \mathscr Q\big|_S$, induced by removing the unobserved variables from distributions and causal graphs.
        \item[2)] We know that the conditions of Theorem \ref{thm:suff-interv} are only satisfied for a subset $    E_{\mathcal E}$ of all edges in $G_Q\big|_S$.
    \end{enumerate}
    We define the restricted IKL divergence via:
    \begin{equation}
        \IKL^\text{res}(\mathscr{P}_{\mathcal E}\lVert \mathscr Q) = \frac{1}{|\mathcal E|}\sum_{e\in \mathcal E}\sum_{i \in {\bf X}_{E_{\mathcal E}} \cap \mathcal I_e^C}\Ex\limits_{P^e({\bf pa}_i^{G_Q|_S})}\left[\KL\left(P^e(X_i \ | \ {\bf pa}_i^{G_Q|_S}) \Vert Q(X_i \ |  \ {\bf pa}_i^{G_Q|_S})\right)\right]
    \end{equation} 
    summing only over un-intervened variables $X$ such that $(X, \cdot)$ or $(\cdot, X)$ is contained in $E_\mathcal E$. 
    Then, $\IKL^\text{res}(\mathscr{P}_{\mathcal E}\lVert \mathscr Q) = 0$ implies that the graphs $G_P, G_Q$ coinicde on the edge set $E_{\mathcal E}$. Conversely, if  $\IKL^\text{res}(\mathscr{P}_{\mathcal E}\lVert \mathscr Q) > 0$, there exists a variable $X_i \in {\bf X}_{E_\mathcal E}$ such that ${\bf PA}_i^{G_P} \neq {\bf PA}_i^{G_Q|_S}$.
\end{restatable} 
\begin{example}\label{example:interventions-not-sufficient}
Let $(P, G_P), (Q, G_Q)$ be two CGMs with $P=Q$ and graphs given in Figure \ref{fig:ex2-DAGS}. We want to compute the interventional KL divergence to evaluate the fit of $(Q, G_Q)$ to $(P, G_P)$ using interventional distributions $P^{e_1}, P^{e_2}$. 
\begin{figure}[h]
\begin{minipage}{0.5\textwidth}
   \centering
   \begin{tikzpicture}[
      node distance=0.3cm and .5cm,
      mynode/.style={draw,circle,align=center, minimum size=0.2cm}
    ]
        \node[mynode] (x1) {$X_1$};
    \node[mynode,right=of x1] (x2) {$X_2$};
    \node[mynode,right=of x2] (x3) {$X_3$};
    \node[mynode,above right=of x3] (x4) {$X_4$};
    \node[mynode,below right=of x3] (x5) {$X_5$};
    \path (x1) edge[-latex] (x2)
          (x2) edge[-latex] (x3)
          (x3) edge[-latex] (x4)
          (x2) edge[bend left][-latex] (x4)
          (x3) edge[-latex] (x5);
    \end{tikzpicture}
\end{minipage}\hfill
\begin{minipage}{0.5\textwidth}
   \centering
   \begin{tikzpicture}[
      node distance=0.3cm and .5cm,
      mynode/.style={draw,circle,align=center, minimum size=0.2cm}
    ]
    \node[mynode] (x1) {$X_1$};
    \node[mynode,right=of x1] (x2) {$X_2$};
    \node[mynode,right=of x2] (x3) {$X_3$};
    \node[mynode,above right=of x3] (x4) {$X_4$};
    \node[mynode,below right=of x3] (x5) {$X_5$};
    \path (x1) edge[-latex] (x2)
          (x2) edge[-latex] (x3)
          (x3) edge[latex-] (x4)
          (x2) edge[bend left][-latex] (x4)
          (x3) edge[-latex] (x5);
    \end{tikzpicture}
\end{minipage}
 \caption{Two Markov equivalent DAGs $G_P$ (left) and $G_Q$ (right) over variables $X_1, \ldots, X_5$.} 
 \label{fig:ex2-DAGS}
\end{figure}
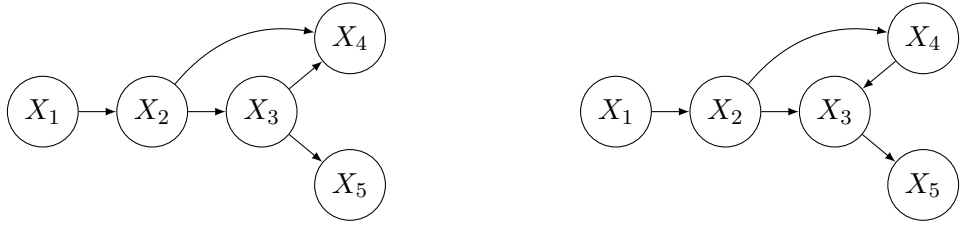
Assume that in environment $e_1$ only variables $X_2, X_3, X_5$ are observed with a mechanism shift $P(X_5|X_3) \mapsto \tilde P(X_5|X_3)$, i.e.\ $E_{e_1} =\{(X_3, X_5)\}$ and ${\bf X}_{e_1} = \{X_3, X_5\}$ in Corollary \ref{cor:identification-distance}. We can then compute the restricted IKL divergence using only the observed variables:
\begin{equation}
    \IKL^\text{res}(\mathscr{P}_{\{e_1\}}\lVert \mathscr Q) = \Ex_{P^{e_1}(X_2)}\KL(P^{e_1}(X_3 | X_2) \lVert Q(X_3 | X_2)) = 0
\end{equation}
\end{example}
indicating that the edge $(X_3, X_5)$ is oriented correctly in $G_Q$. Note that this equality holds even if the unobserved variable $X_4$ also happened to be intervened, since it is not conditioned upon in the restricted IKL divergence.
Next, assume that in another environment $e_2$, variables $ X_2, X_3, X_4$ are observed with a mechanism shift affecting $P(X_3|X_2) \mapsto \tilde P(X_3|X_2)$. Thus, $E_{e_2} = \{(X_4, X_3), \allowbreak (X_2, X_3)\}, {\bf X}_{e_2} = \{X_2, X_3, X_4\}$ and the restricted IKL divergence equals 
\begin{equation}
    \IKL^\text{res}(\mathscr{P}_{\{e_2\}}\lVert \mathscr Q) = \Ex_{P^{e_2}(X_2)}\KL(P^{e_1}(X_4 | X_2) \lVert Q(X_4 | X_2)) +  \KL(P^{e_2}(X_2) \lVert Q(X_2)).
\end{equation}
The first term is strictly greater than zero due to the mis-oriented edge between $X_3$ and $X_4$, i.e.\ case (ii) in Lemma \ref{lem:cond-change} is satisfied. Despite the correct orientation of the edge $X_2 \to X_3$, the second term may also be greater than zero. This happens precisely if the unobserved variable $X_1$ also happened to be intervened (case (ii) of Lemma \ref{lem:cond-change} would again be satisfied). Otherwise, this term would vanish, assuring us that the edge is oriented correctly.
Thus, in the presence of partial observability, nonzero terms in the interventional KL divergence are not necessarily caused by incorrect (observed) edges. Conversely, however, a vanishing restricted IKL divergence guarantees all identifiable edges $E_{\mathcal E}$ to be oriented correctly.

\section{Discussion} \label{sec:discussion}
Multi-environment distributions provide a natural setting for machine learning algorithms to both learn causal structures and leverage them to achieve robustness and out-of-distribution generalization. We have assayed how knowledge about the true causal graph can be used to facilitate learning under distribution shifts and generalize from marginal observations to properties of the joint distribution. Both of these are significant problems for learning in real-world settings, where we cannot always afford to collect a large dataset suited to a task.
Conversely, multi-domain data can provide a useful learning signal to drive causal discovery. 

The focus of the current paper has been a theoretical exploration of properties of the Interventional KL divergence, and applications are thus beyond our scope. Nevertheless, we would like to discuss some possible use cases to point out directions for future work. 

\paragraph{Interventional KL Divergence and Estimation}
The following result captures the operational significance of low IKL divergence.

\begin{theorem}\label{thm:ikl-estim}
   Given causal models $\mathscr P = (P, G_P), \mathscr Q = (Q, G_Q)$ and a set of interventional distributions $\mathscr P_{\mathcal E} = \left((P^e, \mathcal I_e)\right)_{e\in \mathcal E}$, define $Q^e$ for each environment $e\in \mathcal E$ as in Theorem \ref{thm:known-interv}. 
   
   Suppose $\IKL(\mathscr{P}_{\mathcal E}\lVert \mathscr Q)\leq \epsilon$.
      Then, for any bounded function $f$ mapping to $[-B, B]$ and any parameter $\rho > 0$, for at least $1-\rho$ fraction of the environments $e \in \mathcal E$:
   \[
   \left|\Ex_{P^e({\bf X})}[f({\bf X})] - \Ex_{Q^e({\bf X})}[f({\bf X})] \right| \leq \frac{B\sqrt{\epsilon}}{\rho}
   \]
\end{theorem}

\noindent In other words, if the IKL divergence is small, then for most environments, executing the same mechanism changes in `Nature' and in our model $\mathscr{Q}$ would yield similar statistics. Thus, the more heterogeneous the set of environments, the more similar $\mathscr{P}$ and $\mathscr{Q}$ are in terms of their causal implications.
\paragraph{Evaluation tool for causal discovery methods}
If data from multi-environment distributions are available, there is a range of methods to recover the true causal structure asymptotically \citep{causdiscSMS, JANZING20121, Huang_causal_discovery_from_het, multi-env-unknown-he-geng, rojas_invariant_transfer,arjovsky2019invariant,krueger2021out}. Here, the IKL divergence could help develop methods to verify a learned causal model, provided that intervention targets are known for some of the environments. This can then take into account both the interventional differences, as encoded by the learned DAG, and the observational differences encoded in the learned conditional distributions.

\paragraph{Online Learning of Causal Structures}
If we do not have a sufficient number of environment distributions a priori to recover the true causal graph, but intervention targets are more broadly available, we have shown in Corollary \ref{cor:identification-distance} that one could evaluate the fit of a model to the true causal model with respect to subsets of variables. These partially learned and verified models could then be used for inference until more interventional data become available to uncover further causal relationships. Indeed, our metric can even serve for orienting edges in the graph:
If we compute $\IKL(\mathscr P_{\{e_1\}}\lVert (P, G_Q))$, then any non-zero term in the IKL divergence is necessarily caused by mis-oriented edges. Specifically, if for some variable $i\in [d]\backslash \mathcal I_e$, we have  $$\Ex_{P^e({\bf PA}_i^{G_Q})}\left[\KL\left(P^e(X_i \ | \ {\bf PA}_i^{G_Q}) \lVert P(X_i \ | \ {\bf PA}_i^{G_Q})\right)\right] > 0,$$ then ${\bf PA}_i^{G_Q} \neq {\bf PA}_i^{G_P}$.
Thus, we can identify the true parent set by re-computing this term with respect to parent sets ${\bf PA}_i^G$ for all graphs $G$ in the Markov equivalence class of $G_Q$. In this way, the IKL divergence is monotonic w.r.t.\ both distributional differences and structural differences represented by the number of correctly identified edges.

\acks{We thank Armin Keki\'{c}, Julius von K\"ugelgen, Nasim Rahaman and the T\"{u}bingen Causality Team for helpful discussions and comments.}

\appendix

\section{Full proofs}
\label{appendix:proofs}

\subsection{Proof of Lemmas \ref{lem:KL1} and \ref{lem:KL2}}
\lemKLone*
\lemKLtwo*
\begin{proof}
Both of these Lemmas are immediate consequences of Lemma \ref{lem:technical}, noting that $\pi_{G_Q}(P)$ satisfies the condition of $\hat P$ by Lemma \ref{lem:projection}. The decomposition for $\KL(P^e({\bf X}) \lVert Q({\bf X}))$ then just follows by splitting up the sums.

\end{proof}

\subsection{Proof of Theorem \ref{lem:interv-equiv-light}}
\lemintervequivlight*
\begin{proof}
    Since $P, Q$ are both Markovian with respect to $G_P$, it follows by Lemma \ref{lem:KL1} that
    \begin{align}
        P = Q & \Leftrightarrow \KL(P({\bf X}) \lVert Q({\bf X})) = 0 \\
        & \overset{(\ref{eq:kl-decomposition-special-lemma-7})}{\Leftrightarrow}  \Ex_{P({\bf pa}_i^{G_P})}\left[\KL(P(X_i \ | \ {\bf pa}_i^{G_P}) \lVert Q(X_i \ | \ {\bf pa}_i^{G_P}))\right] = 0 \quad \text{for any }i \in [d] \\
        & \Leftrightarrow P(X_i \ | \ {\bf pa}_i^{G_P}) = Q(X_i \ | \ {\bf pa}_i^{G_P}) \quad \text{for any }i \in [d]
    \end{align}
    By our assumption on the multi-environment distributions, we further have that    
    \begin{equation}
         P^e(X_i \ | \ {\bf PA}_i^{G_P}) = P(X_i \ | \ {\bf PA}_i^{G_P}) \quad \text{if and only if} \quad i \in [d] \backslash \mathcal I_e
    \end{equation}
    Concluding, $P=Q$ implies that $P^e(X_i \ | \ {\bf PA}_i^{G_P}) = Q(X_i \ | \ {\bf PA}_i^{G_P})$ for all $i \in [d] \backslash \mathcal I_e$ leading to $ \IKL(\mathscr{P}_{\mathcal E}\lVert \mathscr Q) = 0$. Conversely, assume that $ \IKL(\mathscr{P}_{\mathcal E}\lVert \mathscr Q) = 0$. Since $\bigcap_{e\in \mathcal E} \mathcal I_e = \emptyset$, it follows that $\bigcup_e [d] \backslash \mathcal I_e = [d]$, i.e.\ each term $\Ex_{P^e({\bf pa}_i^{G_P})}\left[\KL(P(X_i \ | \ {\bf pa}_i^{G_P}) \lVert Q(X_i \ | \ {\bf pa}_i^{G_P}))\right]$ is included in $ \IKL(\mathscr{P}_{\mathcal E}\lVert \mathscr Q)$ for some environment $e\in \mathcal E$. From $\IKL(\mathscr{P}_{\mathcal E}\lVert \mathscr Q) = 0$, we can thus deduce that each local divergence vanishes, and $P(X_i \ | \ {\bf pa}_i^{G_P}) = Q(X_i \ | \ {\bf pa}_i^{G_P})$ for any $i \in [d]$ implying that $P=Q$.
\end{proof}

\subsection{Proof of Lemma \ref{lem:projection}}
\lemprojection*
\begin{proof}
    To prove the identity, we make use of the following technical Lemma:
    \begin{lemma}\label{lem:technical}
        For distributions $P, Q$ such that $Q$ is Markovian w.r.t.\ $G$, we have the following decomposition:
        \begin{equation}
            \KL(P({\bf X}) \lVert Q({\bf X})) = \sum_{i=1}^d \Ex\limits_{P({\bf pa}_i^G)}\left[\KL\left(P(X_i \ | \ {\bf pa}_i^G) \Vert Q(X_i \ |  \ {\bf pa}_i^G)\right)\right] + \KL(P({\bf X}) \lVert {\hat P}({\bf X})) 
        \end{equation}
        where ${\hat P}$ is a distribution satisfying 
         \begin{equation}
        {\hat P}({\bf X}) = \prod_{i=1}^d P(X_i \ | \ {\bf PA}_i^G)
    \end{equation}
    \end{lemma}
    \begin{proof}[of Lemma \ref{lem:technical}]
        \begin{align*}\KL(P({\bf X}) \Vert Q({\bf X})) &= \Ex_{P({\bf X})} \left[ \log\frac{P  {\hat P}}{Q {\hat P}}\right] \\
& = \Ex_{P({\bf X})} \left[ \log\frac{{\hat P}}{Q}\right] +  \Ex\limits_{P(\bf{X})} \left[\log\frac{P}{{\hat P}}\right]
\\
& =\Ex_{P({\bf X})} \left[ \log\frac{{\hat P}}{Q}\right] + \KL(P({\bf X}) \Vert {\hat P}({\bf X}))
\end{align*}
We now proceed to analyze the remaining term using the definition of $\hat P$.
\begin{align*}
    \Ex_{P({\bf x})}\left[ \log\frac{{\hat P}}{Q}\right]   & = \Ex_{P({\bf x})} \left[ \log \prod_{i=1}^d\frac{{\hat P}(X_i \ | \ {\bf PA}_i)}{Q(X_i \ | \ {\bf PA}_i)}\right] \\ 
    & = \sum_{i=1}^d \Ex_{P({\bf x})} \left[  \log\frac{P(X_i \ | \ {\bf PA}_i)}{Q(X_i \ | \ {\bf PA}_i)}\right] \\ 
    & = \sum_{i=1}^d \sum_{x_i, {\bf pa}_i}  \log\frac{P(x_i \ | \ {\bf pa}_i)}{Q(x_i \ | \ {\bf pa}_i)} \\
    &\hspace{.5cm} \cdot \sum_{{\bf x}_{[d]} \backslash \{x_i, {\bf pa}_i\}} P(x_i, {\bf pa}_i) \cdot P({\bf x}_{[d]} \backslash \{x_i, {\bf pa}_i\} \ | \ x_i, {\bf pa}_i) \\
    & = \sum_{i=1}^d \sum_{x_i, {\bf pa}_i}  \log\frac{P(x_i \ | \ {\bf pa}_i)}{Q(x_i \ | \ {\bf pa}_i)} \cdot P(x_i \ | \ {\bf pa}_i) \cdot P({\bf pa}_i)  \\
    &\hspace{.5cm}\cdot \underbrace{ \sum_{{\bf x}_{[d]} \backslash \{x_i, {\bf pa}_i\}} P({\bf x}_{[d]} \backslash \{x_i, {\bf pa}_i\} \ | \ x_i, {\bf pa}_i)}_{=1 \text{ for every choice of } (x_i, {\bf pa}_i)} \\
    & = \sum_{i=1}^d \Ex\limits_{P({\bf pa}_i^G)}\left[\KL\left(P(X_i \ | \ {\bf pa}_i^G) \Vert Q(X_i \ |  \ {\bf pa}_i^G)\right)\right] 
\end{align*}
To simplify notation, we here leave away the supscript $G$ for the parents and use sums for the expected values, but the continuous case follows with integrals analogously.
    \end{proof}
We can now apply the decomposition above to obtain the result:
\begin{align*}
    \pi_G(P)({\bf X}) &= \argmin\limits_{\tilde P \text{ Markovian w.r.t. } G} \KL(P({\bf X}) \lVert \tilde P({\bf X})) \\
    & = \argmin\limits_{\tilde P \text{ Markovian w.r.t. } G} \sum_{i=1}^d \Ex\limits_{P({\bf pa}_i^G)}\left[\KL\left(P(X_i \ | \ {\bf pa}_i^G) \Vert \tilde P(X_i \ |  \ {\bf pa}_i^G)\right)\right] + \KL(P({\bf X}) \lVert {\hat P}({\bf X}))\\ 
    & = \argmin\limits_{\tilde P \text{ Markovian w.r.t. } G} \sum_{i=1}^d \Ex\limits_{P({\bf pa}_i^G)}\left[\KL\left(P(X_i \ | \ {\bf pa}_i^G) \Vert \tilde P(X_i \ |  \ {\bf pa}_i^G)\right)\right]
\end{align*}
Note that the second term in line 2 does not depend on $\tilde P$. The claim can then be deduced from noting that $\tilde P({\bf X}) = \prod_{i=1}^d P(X_i | {\bf PA}_i^G)$ sets the sum to zero.
\end{proof}

\subsection{Proof of Theorem \ref{thm:known-interv}}
\thmknowninterv*
\begin{proof}
We show the identity using Lemma \ref{lem:KL2} noting that $Q^e$ is still Markovian with respect to $G_Q$ (although it may no longer be faithful to $G_Q$, e.g.\ if an intervention was hard). Thus, we get the following decomposition
\begin{align*}
    \begin{split}
     &\frac{1}{|\mathcal E|} \sum_{e\in \mathcal E} \KL(P^e({\bf X}) \lVert Q^e({\bf X})) 
     \\ & = \frac{1}{|\mathcal E|} \sum_{e\in \mathcal E} \Bigg[\sum_{i=1}^d \Ex\limits_{P^e({\bf pa}_i^{G_{Q}})}\left[\KL\left(P^e(X_i \ | \ {\bf pa}_i^{G_{Q}}) \Vert Q^e(X_i \ |  \ {\bf pa}_i^{G_{Q}})\right)\right]
     + \KL(P^e({\bf X}) \lVert  \pi_{G_{Q}}(P^e)) \Bigg]
     \end{split}
 \intertext{By assumption, we have that 
 \begin{equation}
     Q^e(X_i \ | \ {\bf PA}_i^{G_Q}) = \begin{cases} 
     P^e(X_i \ | \ {\bf PA}_i^{G_Q}), &  \text{if } i \in \mathcal I_e \\
     Q(X_i \ | \ {\bf PA}_i^{G_Q}) & \text{if } i \in [d]\backslash \mathcal I_e
     \end{cases}
 \end{equation}
 Thus, the terms corresponding to $i \in \mathcal I_e$ vanish yielding
}
  &  = \frac{1}{|\mathcal E|} \sum_{e\in \mathcal E} \Bigg[\sum_{i\in [d] \backslash \mathcal I_e} \Ex\limits_{P^e({\bf pa}_i^{G_Q})}\left[\KL\left(P^e(X_i \ | \ {\bf pa}_i^{G_Q}) \Vert Q(X_i \ |  \ {\bf pa}_i^{G_Q})\right)\right] + \KL\left(P^e({\bf X}) \lVert  \pi_{G_{Q}}(P^e)\right) \Bigg] \\ 
  & =\IKL(\mathscr{P}_{\mathcal E}\lVert \mathscr Q)
\end{align*}
\end{proof}

\subsection{Proof of Lemma \ref{lem:cond-change}}
\lemcondchange*

\begin{proof}
This Lemma follows from Corollary 4.4. by \cite{causdiscSMS} using ${\bf Z} = {\bf PA}_i^{G_Q}$.
If case (i), then there exists an unblocked path to a conditioned child of $X_j$ and thus $e$ is $d$-connected to $X_j$ resulting in a change in the corresponding mechanism by faithfulness. Otherwise, if case (ii), there is an unblocked path to an unconditioned parent of $X_i$ and therefore $e$ is $d$-connected to $X_i$ again resulting in a change in the corresponding mechanism.

\end{proof}
\subsection{Proof of Theorem \ref{thm:suff-interv}}
\thmsuffinterv*
\begin{proof}
    Let's first assume that $(P, G_P) = (Q, G_Q)$. By our assumption on the multi-environment distributions \ref{assum:multi-env} and \ref{asum:pseudo-causal}, we have for all $e\in \mathcal E$
    \begin{equation} \label{eq:lem-fitting-1}
         P^e(X_i \ | \ {\bf PA}_i^{G_Q}) = P^e(X_i \ | \ {\bf PA}_i^{G_P}) = P(X_i \ | \ {\bf PA}_i^{G_P}) = Q(X_i \ | \ {\bf PA}_i^{G_Q})
    \end{equation}
    if and only if $i \in [d] \backslash \mathcal I_e$. Thus, all the conditional divergences in the IKL divergence vanish. Since $G_P = G_Q$, $P$ is Markovian w.r.t. $G_Q$ and so are all interventional distributions. Thus, all the residual terms $\KL(P^e({\bf X}) \lVert \pi_{G_Q}(P^e)({\bf X})) = 0$ and therefore $\IKL(\mathscr P_{\mathcal E} \lVert \mathscr Q) = 0$.
    
    Conversely, assume that $(P, G_P) \neq (Q, G_Q)$. If $P\neq Q$, we have that 
    $\IKL(\mathscr P_{\mathcal E} \lVert \mathscr{Q}) \geq 1/|\mathcal E| \KL(P({\bf X})\lVert Q({\bf X})) > 0$. So it suffices to show that $P=Q$, but $G_P \neq G_Q$ implies that $\IKL(\mathscr{P}_{\mathcal E}\lVert \mathscr Q) > 0$. 
    If $P=Q$, then $G_Q$ is necessarily Markov equivalent to $G_P$ and thus $G_Q, G_P$ share the same skeleton and v-structures. 
    Since $G_P \neq G_Q$, there exist adjacent variables $X_i, X_j$ such that $X_j \overset{G_P}{\to} X_i$, but $X_i \overset{G_Q}{\to} X_j$.
    If condition (i) is satisfied and there is also an unblocked path in $G_P$ $e \overset{G_P}{\to} X_i$ given ${\bf PA}_j^{G_Q}\backslash X_i$, we have by Lemma \ref{lem:cond-change} that $P^e(X_j \ | \ {\bf PA}_j^{G_Q}) \neq P(X_j \ | \ {\bf PA}_j^{G_Q}  ) =  Q(X_j \ | \ {\bf PA}_j^{G_Q})$ which gives us \begin{equation*}
         \Ex_{P^e({\bf PA}_j^{G_Q})} \left[\KL\left(P^e(X_j \ | \ {\bf PA}_j^{G_Q}) \lVert Q(X_j \ | \ {\bf PA}_j^{G_Q})\right)\right] > 0
    \end{equation*}
    so that $\IKL(\mathscr{P}_{\mathcal E}\lVert \mathscr Q)> 0$, since $j\not\in \mathcal I_e$.  
    If otherwise, every such path is blocked in $G_P$, let $(e, X_0, \ldots, X_i)$ be the shortest of such paths in $G_Q$, which implies that
    \begin{itemize}
        \item None of the intermediate variables is intervened
        \item Each variable on the path has exactly one parent in $G_Q$ on the path.
    \end{itemize}
    
     \noindent If $X_0 = X_i$, the path is also unblocked in $G_P$, since $X_i$ is directly intervened upon. So, the length of the path is at least 2. Since this path is blocked in $G_P$, there has to be a mis-oriented edge on the path in $G_P$ compared to $G_Q$. Let $X_k \overset{G_Q}{\to} X_l$ and $X_l \overset{G_P}{\to} X_k$ denote the first of such mis-oriented edges on the path. But then $(e, X_0, \ldots, X_k)$ yields an unblocked path $e \overset{G_P}{\to} X_k$ given ${\bf PA}_l^{G_Q} \backslash X_k$ and $X_l \not \in \mathcal I_e$. Therefore, Lemma \ref{lem:cond-change} applies, leading to a nonzero term in the IKL divergence.   
     If condition (ii) is satisfied, and there is an unblocked path in $G_p$ $e \overset{G_P}{\to} X_j$ given ${\bf PA}_i^{G_Q}$, we have 
     \begin{equation*}
         \Ex_{P^e({\bf PA}_i^{G_Q})}\left[ \KL\left(P^e(X_i \ | \ {\bf PA}_i^{G_Q}) \lVert Q(X_i \ | \ {\bf PA}_i^{G_Q})\right)\right] > 0
    \end{equation*}
     and consequently $\IKL(\mathscr{P}_{\mathcal E}\lVert \mathscr Q)> 0$. Otherwise, we get an unblocked path to an intermediate node in the same way as for case (i).

\end{proof}

\subsection{Proof of Corollary \ref{cor:identification-distance}}
\coridentificationdistance*
\begin{proof}
    Since $P=Q$, we have that $G_P, G_Q$ are Markov equivalent, so they share the same skeleton. In particular, the subgraphs $G_P|_S, G_Q|_S$ have the same skeleton. By definition, $E_{\mathcal E}$ contains exactly those edges for which one of the conditions in Theorem \ref{thm:suff-interv} is satisfied. If one of the edges in $E_{\mathcal E}$ was mis-oriented between $G_P|_S, G_Q|_S$ (and therefore $G_P, G_Q$), this would lead to a nonzero term in the restricted IKL divergence by the same arguments as in proof of Theorem \ref{thm:suff-interv}. 

    Conversely, if we had ${\bf PA}_i^{G_P} = {\bf PA}_i^{G_Q|_S}$ for all variables $X_i$ that occur in the sum of the IKL divergence, by assumptions \ref{assum:multi-env} and \ref{asum:pseudo-causal}, all terms in the IKL would vanish.

\end{proof} 

\subsection{Proof of Theorem \ref{thm:ikl-estim}}
\begin{proof}
    By Theorem \ref{thm:known-interv}:
    \[
    \frac{1}{|\mathcal E|} \sum_{e\in \mathcal E} \KL(P^e({\bf X}) \lVert Q^e({\bf X}))  \leq \epsilon
    \]
    Using Pinsker's inequality and the concavity of the square-root function:
    \begin{align*}
    \frac{1}{|\mathcal E|} \sum_{e\in \mathcal E} \TV(P^e({\bf X}), Q^e({\bf X}))
    &\leq \frac{1}{|\mathcal E|} \sum_{e\in \mathcal E} \sqrt{\KL(P^e({\bf X}), Q^e({\bf X}))/2}\\
    &\leq \sqrt{\frac{1}{|\mathcal E|} \sum_{e\in \mathcal E} \KL(P^e({\bf X}), Q^e({\bf X}))}\\
    &\leq \sqrt{\epsilon}
    \end{align*}
    By Markov's inequality, for at least $1-\rho$ fraction of $e$'s, 
    \[\TV(P^e({\bf X}), Q^e({\bf X}))\leq \frac{\sqrt{\epsilon}}{\rho}.\]
    The claim now follows from the boundedness of $f$.
\end{proof}
\end{document}